\documentclass[final,12pt]{colt2022} 

\title[]{Thoughts on the Consistency between Ricci Flow and Neural Network Behavior}
\usepackage{times}

\coltauthor{\Name{Jun Chen} \Email{junc@zju.edu.cn}\\  \Name{Tianxin Huang} \Email{21725129@zju.edu.cn}\\
\Name{Wenzhou Chen} \Email{wenzhouchen@zju.edu.cn}\\
\Name{Yong Liu} \Email{yongliu@iipc.zju.edu.cn}\\
\addr Institute of Cyber-Systems and Control, Zhejiang University, Hangzhou, China}

\begin{document}

\maketitle

\begin{abstract}%
  The Ricci flow is a partial differential equation for evolving the metric in a Riemannian manifold to make it more regular. On the other hand, neural networks seem to have similar geometric behavior for specific tasks. In this paper, we construct the linearly nearly Euclidean manifold as a background to observe the evolution of Ricci flow and the training of neural networks. Under the Ricci-DeTurck flow, we prove the dynamical stability and convergence of the linearly nearly Euclidean metric for an $L^2$-Norm perturbation. In practice, from the information geometry and mirror descent points of view, we give the steepest descent gradient flow for neural networks on the linearly nearly Euclidean manifold. During the training process of the neural network, we observe that its metric will also regularly converge to the linearly nearly Euclidean metric, which is consistent with the convergent behavior of linearly nearly Euclidean metrics under the Ricci-DeTurck flow.
\end{abstract}

\begin{keywords}%
  Geometry, Ricci flow, Neural network, Optimization, Information geometry
\end{keywords}

\section{Introduction}

On the one hand, the geometrical structure of neural networks trained on the given dataset will gradually become regular. On the other hand, the Ricci flow is a process of ``surgery" on a manifold, which will make the manifold also become regular. Both of them seem to have the same evolutionary goal. This paper mainly focuses on their performance in response to perturbations. By claiming the perturbation in a neural network, we mean that the gradient direction deviates from the steepest descent direction due to the random selection of data and the size of learning step.

In general, a neural network is embedded in the Euclidean space. Since the Euclidean space is flat and fixed, perturbations will all act on the parameter update and become difficult to eliminate. From the duality theory point of view~\citep{amari2000methods,amari2016information}, we can embed the neural network on the Riemannian manifold. If the perturbation is considered to be all acting on the manifold, then the parameter update will not be affected by the perturbation~\citep{martens2020new}. This is an excellent way to deal with perturbations in neural networks. Naturally, we consider embedding a Riemannian manifold close to the Euclidean space for a neural network, i.e., the linearly nearly Euclidean manifold.

In general relativity~\citep{wald2010general}, a complete Riemannian manifold $(\mathcal{M},g)$ endowed with 
a linearly nearly flat spacetime metric $g_{ij}$ is considered for linearized gravity to solve the Newtonian limit. The form of this metric is $g_{ij}= \eta_{ij}+ \gamma_{ij}$, where $\eta_{ij}$ represents a flat Minkowski metric whose background is special relativity and $\gamma_{ij}$ is small from $\eta_{ij}$. An adequate definition of ``smallness" in this context is that the components of $\gamma_{ij}$ are much smaller than $1$ in some global inertial coordinate system of $\eta_{ij}$. Now, let us step out of the physical world and give a similar metric $g_{ij}= \delta_{ij}+ \gamma_{ij}$ in Riemannian $n$-manifold $(\mathcal{M}^n,g)$, i.e. the linearly nearly Euclidean metric, where $\delta_{ij}$ represents a flat Euclidean metric and $\gamma_{ij}$ is small from $\delta_{ij}$.

For the Riemannian $n$-dimensional manifold $(\mathcal{M}^n,g)$ that is isometric to the Euclidean $n$-dimensional space $(\mathbb{R}^n,\delta)$, Schn{\"u}rer et al.~\citep{schnurer2007stability} have showed the stability of Euclidean space under the Ricci flow for a small $C^0$ perturbation. Koch et al.~\citep{koch2012geometric} have given the stability of the Euclidean space along with the Ricci flow in the $L^{\infty}$-Norm. Moreover, for the decay of the  $L^{\infty}$-Norm on Euclidean space, Appleton~\citep{appleton2018scalar} has given the proof of a different method. Considering the stability of integrable and closed Ricci-flat metrics, Sesum~\citep{sesum2006linear} has proved that the convergence rate is exponential because the spectrum of the Lichnerowicz operator is discrete. Furthermore, Deruelle et al.~\citep{deruelle2021stability} have proved that an asymptotically locally Euclidean Ricci-flat metric is dynamically stable under the Ricci flow, for an $L^2 \cap L^{\infty}$ perturbation on non-flat and non-compact Ricci-flat manifolds.

In this paper, we consider a complete Riemannian $n$-dimensional manifold $(\mathcal{M}^n,g)$ endowed with linearly nearly Euclidean metrics $g(t)=\delta+\gamma(t)$. Our main contributions are as follows:

\begin{enumerate}
	\item We prove the stability of linearly nearly Euclidean manifolds under the Ricci-DeTurck flow for an $L^2$-Norm perturbation if initial metrics are integrable and linearly stable, i.e., has a manifold structure of finite dimension. We mean that any Ricci-DeTurck flow which starts from near $g$ exists for all time and converges to a linearly nearly Euclidean metric near $g$. Note that we use the Einstein summation convention and denote generic constants by $C$.
	\item For a neural network, we define and construct linearly nearly Euclidean manifolds based on the information geometry and mirror descent algorithm. Based on a symmetrized convex function, we obtain the linearly nearly Euclidean divergence which is used to calculate the steepest descent gradient in linearly nearly Euclidean manifolds.
	\item On linearly nearly Euclidean manifolds, we yield the weakly and strongly approximated gradient flow of neural networks, respectively.
	\item Experimentally, when we use the weakly approximated steepest descent gradient flow to learn several neural networks on classification tasks, we observe the evolution of its metric is consistent with the convergent behavior of linearly nearly Euclidean metrics under Ricci flow.
\end{enumerate} 

\section{Ricci Flow}
\label{chapter2}

Let us introduce a partial differential equation, the Ricci flow, without explanation. The concept of the Ricci flow first published by Hamilton~\citep{hamilton1982three} on the manifold $\mathcal{M}$ of a time-dependent Riemannian metric $g(t)$ with the initial metric $g_0$:
\begin{equation}
\begin{aligned}
\frac{\partial}{\partial t} g(t) &=-2 \operatorname{Ric}(g(t)) \\
g(0) &=g_{0}
\end{aligned}
\label{ricci}
\end{equation}
where $\operatorname{Ric}$ denotes the Ricci curvature tensor whose definition can be found in Appendix~\ref{app1}.

The purpose of the Ricci flow is to prove Thurston's Geometrization Conjecture and Poincar\'e Conjecture because the Ricci flow is like a surgical scalpel, trimming irregular manifolds into regular manifolds to facilitate observation and discussion~\citep{sheridan2006hamilton}.

In general, in order to possess good geometric and topological properties, we expect the metric to become converge and round with the help of the Ricci flow. ``become round" means that the solution will not shrink to a point but converge to a constant circle. However, in most cases, we do not even know the convergence of the solution and whether the solution will develop a singularity. Next, we will discuss these issues for brevity.

\subsection{Short Time Existence}

To show that there exists a unique solution for a short time, we must check if the system of the Ricci flow is strongly parabolic.

\begin{theorem}
	\label{thm1}
	When $u: \mathcal{M}\times[0, T) \rightarrow \mathcal{E}$ is a time-dependent section of the vector bundle $\mathcal{E}$ where $\mathcal{M}$ is some Riemannian manifold, if the system of the Ricci flow is strongly parabolic at $u_0$ then there exists a solution on some time interval $[0, T)$, and the solution is unique for as long as it exists.
\end{theorem}
\begin{proof}
	The proofs can be found in \citep{ladyzhenskaia1988linear}.
\end{proof}

\begin{definition}
	\label{def1}
	The Ricci flow is strongly parabolic if there exists $\delta > 0$ such that for all covectors $\varphi \neq 0$ and all symmetric $h_{ij}=\frac{\partial g_{ij}(t)}{\partial t} \neq 0$, the principal symbol of $-2\operatorname{Ric}$ satisfies
	\[
	[-2\operatorname{Ric}](\varphi)(h)_{ij} h^{ij}=g^{p q}\left(\varphi_{p} \varphi_{q} h_{i j}+\varphi_{i} \varphi_{j} h_{p q}-\varphi_{q} \varphi_{i} h_{j p}-\varphi_{q} \varphi_{j} h_{i p}\right) h^{i j}>\delta \varphi_{k} \varphi^{k} h_{r s} h^{r s}.
	\]
\end{definition}
Since the inequality cannot always be satisfied, the Ricci flow is not strongly parabolic, which makes us unable to prove the existence of the solution based on Theorem~\ref{thm1}.

It is possible to understand which parts have an impact on its non-parabolic by the linearization of the Ricci curvature tensor.
\begin{lemma}
	\label{le1}
	The linearization of $-2\operatorname{Ric}$ can be rewritten as
	\begin{equation}
	\begin{aligned}
	&D[-2 \operatorname{Ric}](h)_{i j}=g^{p q} \nabla_{p} \nabla_{q} h_{i j}+\nabla_{i} V_{j}+\nabla_{j} V_{i}+O(h_{ij}) \\
	&\operatorname{where} \;\;\;\;V_{i}=g^{p q}\left(\frac{1}{2} \nabla_{i} h_{p q}-\nabla_{q} h_{p i}\right).
	\end{aligned}
	\end{equation}
\end{lemma}
\begin{proof}
	The proofs can be found in Appendix~\ref{app21}.
\end{proof}

The term $O(h_{ij})$ will have no contribution to the principal symbol of $-2 \operatorname{Ric}$, so ignoring it will not affect our discussion of this problem. By carefully observing the above equation, one finds that the impact on the non-parabolic of the Ricci flow comes from the terms in $V$, not the term $g^{p q} \nabla_{p} \nabla_{q} h_{i j}$. The solution is followed by the DeTurck Trick~\citep{deturck1983deforming} that has a time-dependent reparameterization of the manifold:
\begin{equation}
\begin{aligned}
\frac{\partial}{\partial t}\bar{g}(t)&=-2 \operatorname{Ric}(\bar{g}(t))- \mathcal{L}_{\frac{\partial \varphi(t)}{\partial t}} \bar{g}(t) \\
\bar{g}(0) &=\bar{g}_0 + d,
\end{aligned}
\label{deturck}
\end{equation}
where $d$ is a symmetric (0,2)-tensor on $\mathcal{M}$. See Appendix~\ref{app22} for details. By choosing $\frac{\partial \varphi(t)}{\partial t}$ to cancel the effort of the terms in $V$, the reparameterized Ricci flow is strongly parabolic. Thus, one can say that the Ricci-DeTurck flow has a unique solution, the pullback metric, for a short time.

\subsection{Curvature Explosion at Singularity}

In this subsection, we will present the behavior of the Ricci flow in finite time and show that the evolution of the curvature is close to divergence. The core demonstration is followed with Theorem~\ref{thm4}, which requires some other proof as a foreshadowing.

\begin{theorem}
	\label{thm2}
	Given a smooth Riemannian metric $g_0$ on a closed manifold $\mathcal{M}$, there exists a maximal time interval $[0, T)$ such that a solution $g(t)$ of the Ricci flow, with $g(0) = g_0$, exists and is smooth on $[0, T)$, and this solution is unique.
\end{theorem}
\begin{proof}
	The proofs can be found in \citep{sheridan2006hamilton}.
\end{proof}

\begin{theorem}
	\label{thm3}
	Let $\mathcal{M}$ be a closed manifold and $g(t)$ a smooth time-dependent metric on $\mathcal{M}$, defined for $t \in [0, T)$. If there exists a constant $C < \infty$ for all $x \in \mathcal{M}$ such that
	\begin{equation}
	\int_{0}^{T}\left|\frac{\partial}{\partial t} g_x(t)\right|_{g(t)} d t  \leq C,
	\end{equation}
	then the metrics $g(t)$ converge uniformly as $t$ approaches $T$ to a continuous metric $g(T)$ that is uniformly equivalent to $g(0)$ and satisfies
	\[
	e^{-C} g_x(0) \leq g_x(T) \leq e^C g_x(0).
	\]
\end{theorem}
\begin{proof}
	The proofs can be found in Appendix~\ref{app23}.
\end{proof}

\begin{corollary}
	\label{cor1}
	Let $(\mathcal{M}, g(t))$ be a solution of the Ricci flow on a closed manifold. If $|\operatorname{Rm}|_{g(t)}$ is bounded on a finite time $[0, T)$, then $g(t)$ converges uniformly as $t$ approaches $T$ to a continuous metric
	$g(T)$ which is uniformly equivalent to $g(0)$.
\end{corollary}
\begin{proof}
	The bound on $|\operatorname{Rm}|_{g(t)}$ implies one on $|\operatorname{Ric}|_{g(t)}$. Based on Equation~(\ref{ricci}), we can extend the bound on $|\frac{\partial}{\partial t}g(t)|_{g(t)}$. Therefore, we obtain an integral of a bounded quantity over a finite interval is also bounded, by Theorem~\ref{thm3}.
\end{proof}

\begin{theorem}
	\label{thm4}
	If $g_0$ is a smooth metric on a compact manifold $\mathcal{M}$, the Ricci flow with $g(0) = g_0$
	has a unique solution $g(t)$ on a maximal time interval $t\in [0, T)$. If $T < \infty$, then
	\begin{equation}
	\lim _{t \rightarrow T}\left(\sup _{x \in \mathcal{M}}|\operatorname{Rm}_x(t)|\right)=\infty.
	\end{equation}
\end{theorem}
\begin{proof}
	For a contradiction, we assume that $|\operatorname{Rm}_x(t)|$ is bounded by a constant. It follows from Corollary~\ref{cor1} that the metrics $g(t)$ converge uniformly in the norm induced by $g(t)$ to a smooth metric $g(T)$. Based on Theorem~\ref{thm2}, it is possible to find a solution to the Ricci flow on $t \in [0, T)$ because the smooth metric $g(T)$ is uniformly equivalent to initial metric $g(0)$.
	
	Hence, one can extend the solution of the Ricci flow after the time point $t=T$, which is the result for continuous derivatives at $t=T$. This tell us that the time $T$ of existence of the Ricci flow has not been maximal, which contradicts our assumption. In other words, $|\operatorname{Rm}_x(t)|$ is unbounded.
\end{proof}

As approaching the singular time $T$, the Riemann curvature $|\operatorname{Rm}|_{g(t)}$ becomes no longer convergent and tends to explode.

\section{Evolution of Linearly Nearly Euclidean Metrics under the Ricci Flow}
\label{chapter3}

Next, this paper will focus on linearly nearly Euclidean metrics, proving that them can have a good performance in terms of stability, i.e., the convergence of a Ricci-DeTurck flow $\bar{g}(t)$ to a linearly nearly Euclidean metric $\bar{g}(\infty)$. Before that, we have to construct a family $\bar{g}_0$ of linearly nearly Euclidean reference metrics such that $\frac{\partial}{\partial t} \bar{g}_0(t)=O((\bar{g}(t)-\bar{g}_0(t))^2)$. Let
\[
\mathcal{F}=\left\{\bar{g}(t) \in \mathcal{M}^n\;\big|\;2 \operatorname{Ric}(\bar{g}(t))+ \mathcal{L}_{\frac{\partial \varphi(t)}{\partial t}} \bar{g}(t)=0\right\}
\]
be the set of stationary points under the Ricci-DeTurck flow. We are able to establish a manifold
\begin{equation}
\tilde{\mathcal{F}} = \mathcal{F} \cap \mathcal{U}
\label{f}
\end{equation}
where $\mathcal{U}$ is an $L^2$-neighbourhood of integral $\bar{g}_0$.

\subsection{Analysis on Linearly Nearly Euclidean Metrics}

Let us give the definition of linearly nearly Euclidean metrics without further explanation:

\begin{definition}
	\label{def2}
	A complete Riemannian $n$-manifold $(\mathcal{M}^n, g_0)$ is said to be linearly nearly Euclidean with one end of order $\tau > 0$ if there exists a compact set $K \subset \mathcal{M}$, a radius $r$, a point $x$ in $\mathcal{M}$ and a diffeomorphism satisfying $\phi : \mathcal{M} \backslash K \rightarrow (\mathbb{R}^n \backslash B(x,r))/SO(n)$, where $B(x,r)$ is the ball and $SO(n)$ is a finite group acting freely on $\mathbb{R}^n \backslash \{0\}$, then
	\begin{equation}
	\left|\partial^k(\phi_* \gamma_0)\right|_{\delta}=O(r^{-\tau -k}) \;\;\;  \forall k\geq0
	\end{equation}
	holds on $(\mathbb{R}^n \backslash B(x,r))/SO(n)$. $g_0$ can be linearly decomposed into a form containing the Euclidean metric $\delta$:
	\begin{equation}
	g_0(t)=\delta+\gamma_0(t).
	\end{equation}
\end{definition}

In this paper, we consider the linear stability and integrability of the initial metric $g_0$. Fortunately, similar to the proof process of \citep{koiso1983einstein,besse2007einstein}, we can proceed that $(\mathcal{M}^n, g_0)$ is integral and linearly stable.

\begin{definition}
	\label{def3}
	A complete linearly nearly Euclidean $n$-manifold $(\mathcal{M}^n, g_0)$ is said to be linearly stable if the $L^2$ spectrum of the Lichnerowicz operator $L_{g_0}:=\Delta_{g_0}+2\operatorname{Rm}(g_0)*$ is in $(-\infty,0]$ where $\Delta_{g_0}$ is the Laplacian, when $L_{g_0}$ acting on $d_{ij}$ satisfies
	\[
	\begin{aligned}
	L_{g_0}(d)&=\Delta_{g_0}d+2\operatorname{Rm}(g_0)*d \\
	&=\Delta_{g_0}d+2\operatorname{Rm}(g_0)_{iklj}d_{mn}g_0^{km}g_0^{ln}.
	\end{aligned}
	\]
\end{definition}

\begin{definition}
	\label{def4}
	A $n$-manifold $(\mathcal{M}^n, g_0)$ is said to be integrable if a neighbourhood of $g_0$ has a smooth structure.
\end{definition}

\subsection{Short Time Convergence in the $L^2$-Norm}

For convenience, we rewrite the Ricci-DeTurck flow~(\ref{deturck}) in terms of the difference $d(t):=\bar{g}(t)-\bar{g}_0$, such that
\begin{equation}
\begin{aligned}
\frac{\partial}{\partial t} d(t)&=\frac{\partial}{\partial t} \bar{g}(t)=-2 \operatorname{Ric}(\bar{g}(t))+2 \operatorname{Ric}(\bar{g}_0)+\mathcal{L}_{\frac{\partial \varphi'(t)}{\partial t}} \bar{g}_0-\mathcal{L}_{\frac{\partial \varphi(t)}{\partial t}} \bar{g}(t) \\
&=\Delta d(t)+\operatorname{Rm}*d(t)+F_{\bar{g}^{-1}} * \nabla^{\bar{g}_0} d(t) * \nabla^{\bar{g}_0} d(t)+\nabla^{\bar{g}_0}\left(G_{\Gamma(\bar{g}_0)} * d(t) * \nabla^{\bar{g}_0} d(t)\right),
\label{deturck2}
\end{aligned}
\end{equation}
where the tensors $F$ and $G$ depend on $\bar{g}^{-1}$ and $\Gamma(\bar{g}_0)$. Note that $\bar{g}_0$ is a linearly nearly Euclidean metric which satisfies the above formula, where $d_0(t)=\bar{g}_0(t)-\bar{g}_0$, so that $d(t) - d_0(t)=\bar{g}(t)-\bar{g}_0(t)$ holds. Note that $\|\cdot\|_{L^2}$ or $\|\cdot\|_{L^{\infty}}$ denotes the $L^2$-Norm or $L^{\infty}$-Norm with respect to the metric $\bar{g}_0$.

\begin{lemma}
	\label{lem2}
	Let $(\mathcal{M}^n, \bar{g}_0)$ be a complete linearly nearly Euclidean $n$-manifold. If $\bar{g}(0)$ is a metric satisfying $\|\bar{g}(0)-\bar{g}_0 \|_{L^{\infty}} < \epsilon$ where $\epsilon > 0$, then there exists a constant $C < \infty$ and a unique Ricci–DeTurck flow $\bar{g}(t)$ that satisfies
	\[
	\|\bar{g}(t)-\bar{g}_0 \|_{L^{\infty}} < C\|\bar{g}(0)-\bar{g}_0 \|_{L^{\infty}} < C\cdot\epsilon.
	\]
	If a Ricci-DeTurck flow in $\mathcal{B}_{L^{\infty}}(\bar{g}_0,\epsilon)$ for $t\geq1$, there exist constants such that
	\[
	\left\|\nabla^{k}\left(\bar{g}(t)-\bar{g}_{0}\right)\right\|_{L^{\infty}}<C(k) \epsilon, \quad \forall k \in \mathbb{N}.
	\]
\end{lemma}
\begin{proof}
	The similar statement for the case of negative
	Einstein metrics is given in \citep{bamler2010stability}. The proofs can be translated easily to the case of linearly nearly Euclidean metrics by referring the details \citep{bamler2011stability}.
\end{proof}

\begin{lemma}
	\label{lem3}
	Let $(\mathcal{M}^n, \bar{g}_0)$ be a linearly nearly Euclidean $n$-manifold. For a Ricci–DeTurck flow $\bar{g}(t)$ on a maximal time interval $t \in [0, T)$, if it satisfies $\|\bar{g}(0)-\bar{g}_0 \|_{L^{\infty}} < \epsilon$ where $\epsilon > 0$, then there exists a constant $C < \infty$ for $t \in (0, T)$ such that
	\begin{equation}
	\|\bar{g}(t)-\bar{g}_0 \|_{L^2} < C.
	\end{equation}
\end{lemma}
\begin{proof}
	Based on Lemma~\ref{lem2}, we can consider about $\|\bar{g}(t)-\bar{g}_0 \|_{L^2}$. Let $\kappa$ be a function such that $\kappa=1$ on $B(x,r)$, $\kappa=0$ on $\mathcal{M}^n \backslash B(x,2r)$ and $|\nabla \kappa| \leq 2/r$ where $x \in \mathcal{M}^n$ and a radius $r$.
	
	Followed by Equation~(\ref{deturck2}), we obtain
	\[
	\begin{aligned}
	\frac{\partial}{\partial t} \int_{\mathcal{M}}|d(t)|^{2} \kappa^{2} \mathrm{d} \mu \leq & 2 \int_{\mathcal{M}}\left\langle\Delta d(t), \kappa^{2} d(t)\right\rangle \mathrm{d} \mu+C\|\operatorname{Rm}\|_{L^{\infty}} \int_{\mathcal{M}}|d(t)|^{2} \kappa^{2} \mathrm{d} \mu \\
	&+C\|d(t)\|_{L^{\infty}} \int_{\mathcal{M}}|\nabla d(t)|^{2} \kappa^{2} \mathrm{d} \mu+\int_{\mathcal{M}}\left\langle\nabla(G_{\Gamma} * d * \nabla d), d\right\rangle \kappa^{2} \mathrm{d} \mu \\
	\leq &-2 \int_{\mathcal{M}}|\nabla d(t)|^{2} \kappa^{2} \mathrm{d} \mu+C \int_{\mathcal{M}}|\nabla d(t)||d(t)||\nabla \kappa| \kappa \mathrm{d} \mu \\
	&+C\left(\bar{g}_{0}\right) \int_{\mathcal{M}}|d(t)|^{2} \kappa^{2} \mathrm{d} \mu+C\|d(t)\|_{L^{\infty}} \int_{\mathcal{M}}|\nabla d(t)|^{2} \kappa^{2} \mathrm{d} \mu \\
	\leq &(-2+C\cdot\epsilon+C_1) \int_{\mathcal{M}}|\nabla d(t)|^{2} \kappa^{2} \mathrm{d} \mu+C\left(\bar{g}_{0}\right) \int_{\mathcal{M}}|d(t)|^{2} \kappa^{2} \mathrm{d} \mu \\
	&+\frac{1}{C_1} \int_{\mathcal{M}}|d(t)|^{2}|\nabla \kappa|^{2} \mathrm{d} \mu \\
	\leq &\left(C\left(\bar{g}_{0}\right)+\frac{2}{C_1 r^{2}}\right) \int_{B(x,2r)}|d(t)|^{2} \mathrm{d} \mu.
	\end{aligned}
	\]
	Note that we can always find a suitable $C_1$ to make the above formula true. By integration in time $t$, we can further obtain
	\[
	\int_{\mathcal{M}}|d(t)|^{2} \kappa^{2} \mathrm{d} \mu \leq \int_{\mathcal{M}}|d(0)|^{2} \kappa^{2} \mathrm{d} \mu+\left(C\left(\bar{g}_{0}\right)+\frac{2}{C_1 r^{2}}\right) \int_{0}^{t} \int_{B(x,2r)}|d(s)|^{2} \mathrm{d} \mu \mathrm{d} s < \infty.
	\]
	Consequently, we can find a finite ball that satisfies this estimate.
\end{proof}

\begin{corollary}
	\label{cor2}
	Based on Lemma~\ref{lem3}, we further have
	\begin{equation}
	\sup \int_{\mathcal{M}}|d(t)|^{2} \kappa^{2} \mathrm{d} \mu < \infty.
	\end{equation}
\end{corollary}
\begin{proof}
	We obtain
	\[
	\begin{aligned}
	\sup \int_{\mathcal{M}}|d(t)|^{2} \kappa^{2} \mathrm{d} \mu \leq &\sup \int_{\mathcal{M}}|d(0)|^{2} \kappa^{2} \mathrm{d} \mu \\
	&+ N\left(C\left(\bar{g}_{0}\right)+\frac{2}{C_1 r^{2}}\right)\int_{0}^{t}\sup \int_{\mathcal{M}}|d(s)|^{2} \kappa^{2} \mathrm{d} \mu \mathrm{d} s,
	\end{aligned}
	\]
	where each ball of radius $2r$ on $\mathcal{M}$ can be covered by $N$ balls of radius $r$ because $(\mathcal{M}^n, \bar{g}_0)$ is linearly nearly Euclidean. By the Gronwall inequality, we have
	\[
	\sup \int_{\mathcal{M}}|d(t)|^{2} \kappa^{2} \mathrm{d} \mu \leq \exp \left(N\left(C\left(\bar{g}_{0}\right)+\frac{2}{C_1 r^{2}}\right)t\right)\sup \int_{\mathcal{M}}|d(0)|^{2} \kappa^{2} \mathrm{d} \mu.
	\] 
	
	For the $L^2$-Norm, the Ricci-DeTurck flow in linearly nearly Euclidean manifolds has a solution for a short time.
\end{proof}

\subsection{Long Time Stability in the $L^2$-Norm}

Before starting the discussion about long time stability of linearly nearly Euclidean metrics, we need some prior knowledge:

\begin{lemma}
	\label{lem4}
	Let $\bar{g}(t)$ be a Ricci–DeTurck flow on a maximal time interval $t \in (0,T)$ in an $L^2$ neighbourhood of $\bar{g}_0$.
	We have the following estimate such that:
	\[
	\left\|\frac{\partial}{\partial t} d_{0}(t)\right\|_{L^2} \leq C\left\|\nabla^{\bar{g}_{0}(t)}\left(d(t)-d_{0}\right)\right\|_{L^{2}}^{2}.
	\]
\end{lemma}
\begin{proof}
	Let $\{e_1(t),e_2(t),\ldots,e_n(t)\}$ be a family of $L^2$-orthonormal bases of the kernel $\operatorname{ker}_{L^2}$ such that $\frac{\partial}{\partial t}e_i(t)$ depends linearly on $\frac{\partial}{\partial t}d_0(t)$. For an isomorphism orthogonal projection $\Pi: T_{\bar{g}_0} \tilde{\mathcal{F}} \rightarrow \operatorname{ker}_{L^2}$, by the Hardy inequality~\citep{minerbe2009weighted}, one has similar proofs by referring the details~\citep{deruelle2021stability}.
\end{proof}

\begin{theorem}
	\label{thm5}
	Let $(\mathcal{M}^n, \bar{g}_0)$ be a linearly nearly Euclidean $n$-manifold which is linearly stable and integrable. Furthermore, there exists a constant $\alpha_{\bar{g}_0}$ such that
	\[
	\left(\Delta d(t)+\operatorname{Rm}(\bar{g}_0)*d(t), d(t)\right)_{L^{2}} \leq -\alpha_{\bar{g}_0}\left\|\nabla^{\bar{g}_0} h\right\|_{L^{2}}^{2}
	\]
	for all $\bar{g}(t) \in \tilde{\mathcal{F}}$ which is as in (\ref{f}).
\end{theorem}
\begin{proof}
	The similar proofs can be found in \citep{devyver2014gaussian} with some minor modifications. Due to the linear stability requirement of linearly nearly Euclidean manifolds in Definition~\ref{def3}, $-L_{\bar{g}_0}$ is non-negative. Then there exists a positive constant $\alpha_{\bar{g}_0}$ such that
	\[
	\alpha_{\bar{g}_0}\left(-\Delta d(t), d(t)\right)_{L^{2}} \leq \left(-\Delta d(t)-\operatorname{Rm}(\bar{g}_0)*d(t), d(t)\right)_{L^{2}}.
	\]
	By Taylor expansion, one repeatedly uses elliptic regularity and Sobolev embedding~\citep{pacini2010desingularizing} to obtain the estimate.
\end{proof}

\begin{corollary}
	\label{cor3}
	Let $(\mathcal{M}^n, \bar{g}_0)$ be a linearly nearly Euclidean $n$-manifold which is integrable. For a Ricci–DeTurck flow $\bar{g}(t)$ on a maximal time interval $t \in [0, T]$, if it satisfies $\|\bar{g}(t)-\bar{g}_0 \|_{L^{\infty}} < \epsilon$ where $\epsilon > 0$, then there exists a constant $C < \infty$ for $t \in [0, T]$ such that the evolution inequality satisfies
	\[
	\|d(t) - d_{0}\|^2_{L^2} \geq C \int_{0}^{T}\left\|\nabla^{\bar{g}_{0}(t)}\left(d(t)-d_{0}\right)\right\|_{L^{2}}^{2} \mathrm{d} t.
	\]
\end{corollary}
\begin{proof}
The proofs can be found in Appendix~\ref{app4}.
\end{proof}

\begin{theorem}
	\label{thm6}
	Let $(\mathcal{M}^n, \bar{g}_0)$ be a linearly nearly Euclidean $n$-manifold which is linearly stable and integrable. For every $ \epsilon_1 > 0$, there exists a $\epsilon_2 > 0$ satisfying: For any metric $\bar{g}(t) \in \mathcal{B}_{L^2}(\bar{g}_0, \epsilon_2)$, there is a complete Ricci–DeTurck flow $(\mathcal{M}^n, \bar{g}(t))$ starting from $\bar{g}(t)$ converging to a linearly nearly Euclidean metric
	$\bar{g}(\infty) \in \mathcal{B}_{L^2}(\bar{g}_0, \epsilon_1)$. Note that $\mathcal{B}_{L^2}(\bar{g}_0, \epsilon)$ is the $\epsilon$-ball with respect to the $L^2$-Norm induced by $\bar{g}_0$ and centred at $\bar{g}_0$.
\end{theorem}
\begin{proof}
	By Lemma~\ref{lem2}, one can find so small $\epsilon_2>0$ such that $d(t) \in \mathcal{B}_{L^2}(0, \epsilon_2)$ holds. By Lemma~\ref{lem4} and Corollary~\ref{cor3}, we have
	\[
	\begin{aligned}
	&\left\|d_{0}(T)\right\|_{L^{2}} \leq C \int_{1}^{T}\left\|\frac{\partial}{\partial t} d_{0}(t)\right\|_{L^{2}} \mathrm{d} t \\
	&\quad \leq C \int_{1}^{T}\left\|\nabla^{\bar{g}_{0}}\left(d(t)-d_{0}(t)\right)\right\|_{L^{2}}^{2} \mathrm{d} t \\
	&\quad \leq C\left\|d(1)-d_{0}(1)\right\|_{L^{2}}^{2} \leq C\|d(1)\|_{L^{2}}^{2} \leq C \cdot\left(\epsilon_2\right)^{2}.
	\end{aligned}
	\]
	Furthermore, we obtain
	\[
	\left\|d(T)-d_0(T)\right\|_{L^{2}} \leq \|d(1)-d_0(1)\|_{L^{2}} \leq C \cdot \epsilon_2.
	\]
	By the triangle inequality, we get
	\[
	\left\|d(T)\right\|_{L^{2}} \leq C \cdot\left(\epsilon_2\right)^{2} + C \cdot \epsilon_2.
	\]
	Followed by Corollary~\ref{cor2} and Lemma~\ref{lem4}, such $T$ should be pushed further outward, because
	\[
	\lim_{t \rightarrow +\infty}\sup\left\|\frac{\partial}{\partial t} d_{0}(t)\right\|_{L^{2}} \leq \lim_{t \rightarrow +\infty}\sup\left\|\nabla^{\bar{g}_{0}}\left(d(t)-d_{0}(t)\right)\right\|_{L^{2}}^{2}=0.
	\]
	Thus, as $t$ approaches to $+\infty$, $\bar{g}(t)$ converges to $\bar{g}(\infty)=\bar{g}_0+d_0(\infty)$. By the Euclidean Sobolev inequality~\citep{minerbe2009weighted}, $d(t)-d_0(t)$ converges to $0$ as $t$ goes to $+\infty$,
	\[
	\lim_{t \rightarrow +\infty}\left\|d(t)-d_{0}(t)\right\|_{L^{2}} \leq \lim_{t \rightarrow +\infty}C\left\|\nabla^{\bar{g}_{0}}\left(d(t)-d_{0}(t)\right)\right\|_{L^{2}}=0.
	\]
	
	We now conclude a result for linearly nearly Euclidean manifolds under the Ricci-DeTurck flow, which ensures a infinite time existence.
\end{proof}

Therefore, we yield the stability of linearly nearly Euclidean manifolds under the Ricci-DeTurck flow, i.e., any Ricci-DeTurck flow starting close to the linearly nearly Euclidean metric will converge to the linearly nearly Euclidean metric.

\section{Gradient Flow of Linearly Nearly Euclidean Manifolds on Neural Networks}

Now, we have clarified the convergence of linearly nearly Euclidean manifolds under the Ricci-DeTurck flow. Furthermore, we will consider the solution of gradient flow on linearly nearly Euclidean manifolds for neural networks, which will allow us to observe the gradient flow behavior of neural network against perturbations. Empirically, we introduce information geometry~\citep{amari2000methods,amari2016information} and mirror descent algorithm~\citep{bubeck2015convex} to construct the gradient flow with the help of divergences.

\subsection{Linearly Nearly Euclidean Divergence}

From the perspective of information geometry and mirror descent algorithm, the metric $\bar{g}$ can be deduced by the divergence that needs to satisfy certain criteria~\citep{basseville2013divergence}. Consequently, we consider two nearby points $P$ and $Q$ in a manifold $\mathcal{M}$, where these two points are expressed in coordinates as $\boldsymbol{\xi}_P$ and $\boldsymbol{\xi}_Q$, where $\boldsymbol{\xi}$ is a column vector. Moreover, the divergence is defined as half the square of an infinitesimal distance between two sufficiently close points in Definition~\ref{def5}.

\begin{definition}
	\label{def5}
	$D[P:Q]$ is called a divergence when it satisfies the following criteria:
	
	(1) $D[P:Q] \geq 0$,
	(2) $D[P:Q]=0$ when and only when $P=Q$, 
	(3) When $P$ and $Q$ are sufficiently close, by denoting their coordinates by $\boldsymbol{\xi}_P$ and $\boldsymbol{\xi}_Q=\boldsymbol{\xi}_P+d\boldsymbol{\xi}$, the Taylor expansion of $D$ is written as
	\[
	D[\boldsymbol{\xi}_P:\boldsymbol{\xi}_P+d\boldsymbol{\xi}]=\frac{1}{2}\sum_{i,j} \bar{g}_{ij}(\boldsymbol{\xi}_P) d\xi_i d\xi_j + O(|d\boldsymbol{\xi}|^3),
	\]
	and metric $\bar{g}_{ij}$ is positive-definite, depending on $\boldsymbol{\xi}_P$.
\end{definition}

Based on our conjecture in introduction, by transferring the perturbation on the neural network to the neural manifold~\citep{martens2020new} (In this paper, we define as the linearly nearly Euclidean manifold), we can observe the performance of the linearly nearly Euclidean metric in response to the perturbation during the training. In order to construct a linearly nearly Euclidean manifold endowed with a linearly nearly Euclidean metric for the neural network, according to Definition~\ref{def5}, one can introduce the divergence to obtain the expression of the metric. And the advantage is that the constructed divergence can be used to calculate the gradient flow of neural networks on linearly nearly Euclidean manifolds. With the assist of Definition~\ref{def6}, we introduce a symmetrized convex function to construct the needed divergence:
\begin{equation}
\phi(\boldsymbol{\xi})=\sum_i \frac{1}{\tau^2} \log \frac{1}{2}\left(\exp(\tau\xi_i)+\exp(-\tau\xi_i)\right)=\sum_i \frac{1}{\tau^2}\log \left(\cosh(\tau\xi_i)\right)
\label{convex}
\end{equation}
where $\tau$ is a constant parameter.
\begin{definition}
	\label{def6}
	The Bregman divergence~\citep{bregman1967relaxation} $D_B[\boldsymbol{\xi}:\boldsymbol{\xi}']$ is defined as the difference between a convex function $\phi(\boldsymbol{\xi})$ and its tangent hyperplane $z=\phi(\boldsymbol{\xi}')+(\boldsymbol{\xi}-\boldsymbol{\xi}')\nabla\phi(\boldsymbol{\xi}')$, depending on the Taylor expansion at the point $\boldsymbol{\xi}'$:
	\[
	D_B[\boldsymbol{\xi}:\boldsymbol{\xi}']=\phi(\boldsymbol{\xi})-\phi(\boldsymbol{\xi}')-(\boldsymbol{\xi}-\boldsymbol{\xi}')\nabla\phi(\boldsymbol{\xi}').
	\]
\end{definition}

\begin{theorem}
	\label{thm7}
	For a convex function $\phi$ defined by Equation~(\ref{convex}), the linearly nearly Euclidean divergence between two points $\boldsymbol{\xi}$ and $\boldsymbol{\xi}'$ is
	\begin{equation}
	D_{LNE}[\boldsymbol{\xi}':\boldsymbol{\xi}] =\sum_i \left[\frac{1}{\tau^2}\log\frac{\cosh(\tau\xi'_i)}{\cosh(\tau\xi_i)} - \frac{1}{\tau}(\xi'_i-\xi_i)\tanh(\tau\xi_i)\right]
	\end{equation}
	where the Riemannian metric is
	\begin{equation}
	\begin{aligned}
	&\bar{g}_{ij}(\boldsymbol{\xi}(t)) = \delta_{ij} - \left[\tanh(\tau\boldsymbol{\xi})\tanh(\tau\boldsymbol{\xi})^\top\right]_{ij} \\
	&=\begin{bmatrix} 1-\tanh (\tau \xi_1(t))\tanh (\tau\xi_1(t))& \cdots &-\tanh (\tau\xi_1(t))\tanh (\tau \xi_n(t))\\ \vdots&\ddots&\vdots \\
	-\tanh (\tau\xi_n(t))\tanh (\tau \xi_1(t))&\cdots&1-\tanh (\tau \xi_n(t))\tanh (\tau\xi_n(t))
	\end{bmatrix}.
	\end{aligned}
	\end{equation}
\end{theorem}
\begin{proof}
	The proofs can be found in Appendix~\ref{app31}.
\end{proof}

Based on Theorem~\ref{thm7}, the form of the metrics $\bar{g}(t)$ constructed by the linearly nearly Euclidean divergence is consistent with the definition of linearly nearly Euclidean metrics, as long as we adjust parameter $\tau$ to satisfy Definition~\ref{def2}. Moreover, we have also proven that the linearly nearly Euclidean divergence satisfies the criteria of
divergence followed by Definition~\ref{def5}.

\subsection{Weak Approximation of the Gradient Flow}

By the linearly nearly Euclidean divergence, we consider the gradient flow for neural networks toward the steepest descent direction on the manifold endowed with linearly nearly Euclidean metrics:

\begin{lemma}
	\label{lem5}
	The steepest descent gradient flow measured by the linearly nearly Euclidean divergence is defined as
	\begin{equation}
	\tilde{\partial}_{\boldsymbol{\xi}} =\bar{g}^{-1}(t) \partial_{\boldsymbol{\xi}} =\left[\delta_{ij} - \tanh(\tau\boldsymbol{\xi}(t))\tanh(\tau\boldsymbol{\xi}(t))^\top\right]^{-1} \partial_{\boldsymbol{\xi}}.
	\end{equation}
\end{lemma}
\begin{proof}
	The proofs can be found in Appendix~\ref{app32}.
\end{proof}

However, Lemma~\ref{lem5} involves inversion, which greatly consumes computing resources. In particular, we propose two methods for approximating the gradient flow: weak approximation and strong approximation respectively.

For the weak approximation of the gradient flow, we put forward higher requirements for such metrics on the basis of Definition~\ref{def2}, which requires that the linearly nearly Euclidean metric is also a strictly diagonally-dominant matrix based on Corollary~\ref{cor4}.

\begin{corollary}
	\label{cor4}
	The weak approximation of the gradient flow measured by the linearly nearly Euclidean divergence is defined as
	\begin{equation}
	\tilde{\partial}_{\boldsymbol{\xi}} \approx \left[\delta_{ij} + \tanh(\tau\boldsymbol{\xi}(t))\tanh(\tau\boldsymbol{\xi}(t))^\top\right] \partial_{\boldsymbol{\xi}}
	\end{equation}
	if the metric satisfies strictly diagonally-dominant.
\end{corollary}
\begin{proof}
	The proofs can be found in Appendix~\ref{app33}.
\end{proof}

{\bf Remark.} We apply the weakly approximated gradient flow of Corollary~\ref{cor4} for the training and observe the neural network behavior. This part of the experiment is in the Appendix~\ref{app5}.

\subsection{Strong Approximation with Neural Networks}

On the other hand, bypassing the requirement of weak approximation in Corollary~\ref{cor4}, our goal is to approximate the gradient flow, $\bar{g}^{-1}(t) \partial_{\boldsymbol{\xi}}$ in Lemma~\ref{lem5}, from the assist of multi-layer perceptron (MLP) neural network. Using the neural network, we can easily yield the strong approximation of the gradient flow because a neural network with a single hidden layer and a finite number of neurons can be used to approximate a continuous function on compact subsets~\citep{jejjala2020neural}, which is stated by the universal approximation theorem~\citep{cybenko1989approximation,hornik1991approximation}.

As an $n \times n$ symmetric matrix, $\bar{g}(t)$ can be decomposed in terms of the combination of entries $P$ and $A$, where $P$ is the entries made up of the elements of the lower triangular matrix that contains $n(n-1)/2$ real parameters and $A$ is the entries made up of the elements of the diagonal matrix that contains $n$ real parameters. During the training in Figure~\ref{approx}, $\tilde{g}(t)$ can be used as strong approximation of $\bar{g}^{-1}(t)$ in the gradient flow.

\begin{figure}[h]
	\centering
	\includegraphics[width=.9\textwidth]{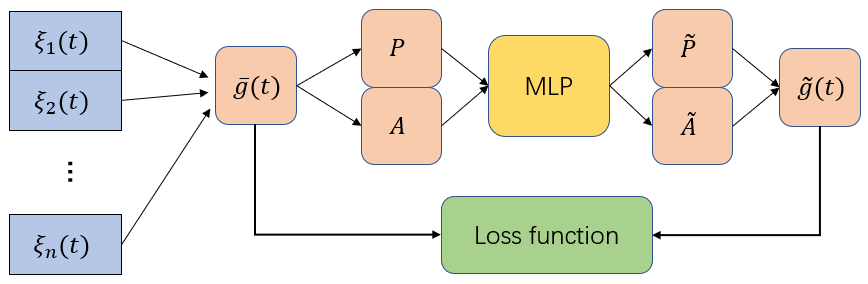}
	\caption{Flow chart of strong approximation of the gradient flow. The new entries $\tilde{P}$ and $\tilde{A}$ produced by neural network get combined into a new metric $\tilde{g}(t)$ that is used to minimize the loss function by combining with the metric $\bar{g}(t)$, where the loss function is defined by Equation~(\ref{loss}).}
	\label{approx}
\end{figure}

\begin{equation}
\mathbb{L}=\|\boldsymbol{I}-\bar{g}(t)\tilde{g}(t)\|^2.
\label{loss}
\end{equation}

\section{Conclusion}

In this paper, we analyze the evolution of linearly nearly Euclidean metrics under the Ricci-DeTurck flow, including proof of convergence in short and infinite time. Furthermore, we construct a linearly nearly Euclidean metric for the neural network with the assist of the information geometry and mirror descent algorithm. {\bf In view of the convergence and stability of linearly nearly Euclidean metrics against perturbations, we observe that the training of the neural network under the weakly approximated gradient flow is consistent with the evolution of the Ricci flow}. In terms of consistency, we hope that this paper will open an exciting future direction which use Ricci flow to assist neural network training in a manifold, similar to the natural gradient.

\acks{We thank a bunch of people and funding agency.}

\bibliography{yourbibfile}

\begin{thebibliography}{30}
\providecommand{\natexlab}[1]{#1}
\providecommand{\url}[1]{\texttt{#1}}
\expandafter\ifx\csname urlstyle\endcsname\relax
  \providecommand{\doi}[1]{doi: #1}\else
  \providecommand{\doi}{doi: \begingroup \urlstyle{rm}\Url}\fi

\bibitem[Amari and Nagaoka(2000)]{amari2000methods}
S-i Amari and H~Nagaoka.
\newblock Methods of information geometry, volume 191 of translations of
  mathematical monographs, s. kobayashi and m. takesaki, editors.
\newblock \emph{American Mathematical Society, Providence, RI, USA}, pages
  2--19, 2000.

\bibitem[Amari(1998)]{amari1998natural}
Shun-Ichi Amari.
\newblock Natural gradient works efficiently in learning.
\newblock \emph{Neural computation}, 10\penalty0 (2):\penalty0 251--276, 1998.

\bibitem[Amari(2016)]{amari2016information}
Shun-ichi Amari.
\newblock \emph{Information geometry and its applications}, volume 194.
\newblock Springer, 2016.

\bibitem[Appleton(2018)]{appleton2018scalar}
Alexander Appleton.
\newblock Scalar curvature rigidity and ricci deturck flow on perturbations of
  euclidean space.
\newblock \emph{Calculus of Variations and Partial Differential Equations},
  57\penalty0 (5):\penalty0 1--23, 2018.

\bibitem[Bamler(2010)]{bamler2010stability}
Richard~H Bamler.
\newblock Stability of hyperbolic manifolds with cusps under ricci flow.
\newblock \emph{arXiv preprint arXiv:1004.2058}, 2010.

\bibitem[Bamler(2011)]{bamler2011stability}
Richard~Heiner Bamler.
\newblock \emph{Stability of Einstein metrics of negative curvature}.
\newblock Princeton University, 2011.

\bibitem[Basseville(2013)]{basseville2013divergence}
Mich{\'e}le Basseville.
\newblock Divergence measures for statistical data processing—an annotated
  bibliography.
\newblock \emph{Signal Processing}, 93\penalty0 (4):\penalty0 621--633, 2013.

\bibitem[Besse(2007)]{besse2007einstein}
Arthur~L Besse.
\newblock \emph{Einstein manifolds}.
\newblock Springer Science \& Business Media, 2007.

\bibitem[Bregman(1967)]{bregman1967relaxation}
Lev~M Bregman.
\newblock The relaxation method of finding the common point of convex sets and
  its application to the solution of problems in convex programming.
\newblock \emph{USSR computational mathematics and mathematical physics},
  7\penalty0 (3):\penalty0 200--217, 1967.

\bibitem[Bubeck et~al.(2015)]{bubeck2015convex}
S{\'e}bastien Bubeck et~al.
\newblock Convex optimization: Algorithms and complexity.
\newblock \emph{Foundations and Trends{\textregistered} in Machine Learning},
  8\penalty0 (3-4):\penalty0 231--357, 2015.

\bibitem[Cybenko(1989)]{cybenko1989approximation}
George Cybenko.
\newblock Approximation by superpositions of a sigmoidal function.
\newblock \emph{Mathematics of control, signals and systems}, 2\penalty0
  (4):\penalty0 303--314, 1989.

\bibitem[Deruelle and Kr{\"o}ncke(2021)]{deruelle2021stability}
Alix Deruelle and Klaus Kr{\"o}ncke.
\newblock Stability of ale ricci-flat manifolds under ricci flow.
\newblock \emph{The Journal of Geometric Analysis}, 31\penalty0 (3):\penalty0
  2829--2870, 2021.

\bibitem[DeTurck(1983)]{deturck1983deforming}
Dennis~M DeTurck.
\newblock Deforming metrics in the direction of their ricci tensors.
\newblock \emph{Journal of Differential Geometry}, 18\penalty0 (1):\penalty0
  157--162, 1983.

\bibitem[Devyver(2014)]{devyver2014gaussian}
Baptiste Devyver.
\newblock A gaussian estimate for the heat kernel on differential forms and
  application to the riesz transform.
\newblock \emph{Mathematische Annalen}, 358\penalty0 (1):\penalty0 25--68,
  2014.

\bibitem[Glass et~al.(2020)Glass, Spasov, and Li{\`o}]{glass2020riccinets}
Samuel Glass, Simeon Spasov, and Pietro Li{\`o}.
\newblock Riccinets: Curvature-guided pruning of high-performance neural
  networks using ricci flow.
\newblock \emph{arXiv preprint arXiv:2007.04216}, 2020.

\bibitem[Hamilton et~al.(1982)]{hamilton1982three}
Richard~S Hamilton et~al.
\newblock Three-manifolds with positive ricci curvature.
\newblock \emph{J. Differential geom}, 17\penalty0 (2):\penalty0 255--306,
  1982.

\bibitem[He et~al.(2016)He, Zhang, Ren, and Sun]{he2016deep}
Kaiming He, Xiangyu Zhang, Shaoqing Ren, and Jian Sun.
\newblock Deep residual learning for image recognition.
\newblock In \emph{Proceedings of the IEEE conference on computer vision and
  pattern recognition}, pages 770--778, 2016.

\bibitem[Hornik(1991)]{hornik1991approximation}
Kurt Hornik.
\newblock Approximation capabilities of multilayer feedforward networks.
\newblock \emph{Neural networks}, 4\penalty0 (2):\penalty0 251--257, 1991.

\bibitem[Jejjala et~al.(2020)Jejjala, Pena, and Mishra]{jejjala2020neural}
Vishnu Jejjala, Damian Kaloni~Mayorga Pena, and Challenger Mishra.
\newblock Neural network approximations for calabi-yau metrics.
\newblock \emph{arXiv preprint arXiv:2012.15821}, 2020.

\bibitem[Koch and Lamm(2012)]{koch2012geometric}
Herbert Koch and Tobias Lamm.
\newblock Geometric flows with rough initial data.
\newblock \emph{Asian Journal of Mathematics}, 16\penalty0 (2):\penalty0
  209--235, 2012.

\bibitem[Koiso(1983)]{koiso1983einstein}
Norihito Koiso.
\newblock Einstein metrics and complex structures.
\newblock \emph{Inventiones mathematicae}, 73\penalty0 (1):\penalty0 71--106,
  1983.

\bibitem[Krizhevsky et~al.(2009)Krizhevsky, Hinton,
  et~al.]{krizhevsky2009learning}
Alex Krizhevsky, Geoffrey Hinton, et~al.
\newblock Learning multiple layers of features from tiny images.
\newblock 2009.

\bibitem[Ladyzhenskaia et~al.(1988)Ladyzhenskaia, Solonnikov, and
  Ural'tseva]{ladyzhenskaia1988linear}
Olga~Aleksandrovna Ladyzhenskaia, Vsevolod~Alekseevich Solonnikov, and Nina~N
  Ural'tseva.
\newblock \emph{Linear and quasi-linear equations of parabolic type},
  volume~23.
\newblock American Mathematical Soc., 1988.

\bibitem[Martens(2020)]{martens2020new}
James Martens.
\newblock New insights and perspectives on the natural gradient method.
\newblock \emph{Journal of Machine Learning Research}, 21:\penalty0 1--76,
  2020.

\bibitem[Minerbe(2009)]{minerbe2009weighted}
Vincent Minerbe.
\newblock Weighted sobolev inequalities and ricci flat manifolds.
\newblock \emph{Geometric and Functional Analysis}, 18\penalty0 (5):\penalty0
  1696--1749, 2009.

\bibitem[Pacini(2010)]{pacini2010desingularizing}
Tommaso Pacini.
\newblock Desingularizing isolated conical singularities: uniform estimates via
  weighted sobolev spaces.
\newblock \emph{arXiv preprint arXiv:1005.3511}, 2010.

\bibitem[Schn{\"u}rer et~al.(2007)Schn{\"u}rer, Schulze, and
  Simon]{schnurer2007stability}
Oliver~C Schn{\"u}rer, Felix Schulze, and Miles Simon.
\newblock Stability of euclidean space under ricci flow.
\newblock \emph{arXiv preprint arXiv:0706.0421}, 2007.

\bibitem[Sesum(2006)]{sesum2006linear}
Natasa Sesum.
\newblock Linear and dynamical stability of ricci-flat metrics.
\newblock \emph{Duke Mathematical Journal}, 133\penalty0 (1):\penalty0 1--26,
  2006.

\bibitem[Sheridan and Rubinstein(2006)]{sheridan2006hamilton}
Nick Sheridan and Hyam Rubinstein.
\newblock Hamilton’s ricci flow.
\newblock \emph{Honour thesis}, 2006.

\bibitem[Wald(2010)]{wald2010general}
Robert~M Wald.
\newblock \emph{General relativity}.
\newblock University of Chicago press, 2010.

\end{thebibliography}

\appendix

\section{Differential Geometry}
\label{app1}
1. Riemann curvature tensor (Rm) is a (1,3)-tensor defined for a 1-form $\omega$:
\[
R^l_{ijk}\omega_l=\nabla_i \nabla_j \omega_k - \nabla_j \nabla_i \omega_k
\]
where the covariant derivative of $F$ satisfies
\[
\nabla_{p} F_{i_{1} \ldots i_{k}}^{j_{1} \ldots j_{l}}=\partial_{p} F_{i_{1} \ldots i_{k}}^{j_{1} \ldots j_{l}}+\sum_{s=1}^{l} F_{i_{1} \ldots i_{k}}^{j_{1} \ldots q \ldots j_{l}} \Gamma_{p q}^{j_{s}}-\sum_{s=1}^{k} F_{i_{1} \ldots q_{\ldots} i_{k}}^{j_{1} \ldots j_{l}} \Gamma_{p i_{s}}^{q}.
\]
In particular, coordinate form of the Riemann curvature tensor is:
\[
R_{i j k}^{l}=\partial_{i} \Gamma_{j k}^{l}-\partial_{j} \Gamma_{i k}^{l}+\Gamma_{j k}^{p} \Gamma_{i p}^{l}-\Gamma_{i k}^{p} \Gamma_{j p}^{l}.
\]
2. Christoffel symbol in terms of an ordinary derivative operator is:
\[
\Gamma^k_{ij}=\frac{1}{2}g^{kl}(\partial_i g_{jl}+\partial_j g_{il}-\partial_l g_{ij}).
\]
3. Ricci curvature tensor (Ric) is a (0,2)-tensor:
\[
R_{ij}=R^p_{pij}.
\]
4. Scalar curvature is the trace of the Ricci curvature tensor:
\[
R=g^{ij}R_{ij}.
\]
5. Lie derivative of $F$ in the direction $\frac{d \varphi(t)}{dt}$:
\[
\mathcal{L}_{\frac{d \varphi(t)}{dt}} F=\left(\frac{d}{dt}\varphi^*(t) F\right)_{t=0}
\]
where $\varphi(t): \mathcal{M} \rightarrow \mathcal{M}$ for $t\in(-\epsilon,\epsilon)$ is a time-dependent diffeomorphism of $\mathcal{M}$ to $\mathcal{M}$.

\section{Proof for the Ricci Flow}
\label{app2}

\subsection{Proof for Lemma~\ref{le1}}
\label{app21}

\begin{definition}
	\label{linear}
	The linearization of the Ricci curvature tensor is given by
	\[
	D[\operatorname{Ric}](h)_{i j}=-\frac{1}{2} g^{p q}(\nabla_{p} \nabla_{q} h_{i j}+\nabla_{i} \nabla_{j} h_{p q}-\nabla_{q} \nabla_{i} h_{jp}-\nabla_{q} \nabla_{j} h_{i p}).
	\]
\end{definition}
\begin{proof}
	Based on Appendix~\ref{app1}, we have
	\[
	\nabla_{q} \nabla_{i} h_{j p} =\nabla_{i} \nabla_{q} h_{j p}-R_{q i j}^{r} h_{r p}-R_{q i p}^{r} h_{j m}.
	\]
	Combining with Definition~\ref{linear}, we can obtain the deformation equation because of $\nabla_k g_{ij}=0$,
	\[
	\begin{aligned}
	D[-2 \mathrm{Ric}](h)_{i j}=& g^{p q} \nabla_{p} \nabla_{q} h_{i j}+\nabla_{i}\left(\frac{1}{2} \nabla_{j} h_{p q}-\nabla_{q} h_{j p}\right)+\nabla_{j}\left(\frac{1}{2} \nabla_{i} h_{p q}-\nabla_{q} h_{i p}\right)+O(h_{ij}) \\
	=& g^{p q} \nabla_{p} \nabla_{q} h_{i j}+\nabla_{i} V_{j}+\nabla_{j} V_{i}+O(h_{ij}).
	\end{aligned}
	\]
\end{proof}

\subsection{Description of the DeTurck Trick}
\label{app22}

Using a time-dependent diffeomorphism $\varphi(t)$, we express the pullback metrics $\bar{g}(t)$:
\[
g(t)=\varphi^*(t) \bar{g}(t)
\]
is a solution of the Ricci flow. Based on the chain rule for the Lie derivative in Appendix~\ref{app1}, we can calculate
\[
\begin{aligned}
\frac{\partial}{\partial t} g(t) &=\frac{\partial\left(\varphi^{*}(t) \bar{g}(t)\right)}{\partial t} \\
&=\left(\frac{\partial\left(\varphi^{*}(t+\tau) \bar{g}(t+\tau)\right)}{\partial \tau}\right)_{\tau=0} \\
&=\left(\varphi^{*}(t) \frac{\partial \bar{g}(t+\tau)}{\partial \tau}\right)_{\tau=0}+\left(\frac{\partial\left(\varphi^{*}(t+\tau) \bar{g}(t)\right)}{\partial \tau}\right)_{\tau=0} \\
&=\varphi^{*}(t) \frac{\partial}{\partial t}\bar{g}(t)+\varphi^{*}(t) \mathcal{L}_{\frac{\partial \varphi(t)}{\partial t}} \bar{g}(t).
\end{aligned}
\]
With the help of Equation~(\ref{ricci}), for the reparameterized metric, we have
\[
\frac{\partial}{\partial t} g(t)=\varphi^{*}(t) \frac{\partial}{\partial t}\bar{g}(t)+\varphi^{*}(t) \mathcal{L}_{\frac{\partial \varphi(t)}{\partial t}} \bar{g}(t)=-2 \operatorname{Ric}(\varphi^*(t) \bar{g}(t))=-2 \varphi^*(t) \operatorname{Ric}(\bar{g}(t)).
\]
The diffeomorphism invariance of the Ricci curvature tensor is used in the last step. The above equation is equivalent to
\[
\frac{\partial}{\partial t}\bar{g}(t)=-2 \operatorname{Ric}(\bar{g}(t))- \mathcal{L}_{\frac{\partial \varphi(t)}{\partial t}} \bar{g}(t).
\]

\subsection{Proof for Theorem~\ref{thm3}}
\label{app23}

Considering any $x \in \mathcal{M}$, $t_0 \in [0, T)$, $V \in T_x \mathcal{M}$, we have
\[
\begin{aligned}
\left|\log \left(\frac{g_x (t_0)(V, V)}{g_x (0)(V, V)}\right)\right| &=\left|\int_{0}^{t_{0}} \frac{\partial}{\partial t}\left[\log g_x (t)(V, V)\right] d t\right| \\
&=\left|\int_{0}^{t_{0}} \frac{\frac{\partial}{\partial t} g_x (t)(V, V)}{g_x (t)(V, V)} d t\right| \\
& \leq \int_{0}^{t_{0}}\left|\frac{\partial}{\partial t} g_x (t) \left(\frac{V}{|V|_{g(t)}}, \frac{V}{|V|_{g(t)}}\right)\right| d t \\
& \leq \int_{0}^{t_{0}}\left|\frac{\partial}{\partial t} g_x (t)\right|_{g(t)} d t \\
& \leq C.
\end{aligned}
\]
By exponentiating both sides of the above inequality, we have
\[
e^{-C} g_x(0)(V, V) \leq g_x(t_0)(V, V) \leq e^C g_x(0)(V, V).
\]
This inequality can be rewritten as
\[
e^{-C} g_x(0) \leq g_x(t_0)(V, V) \leq e^C g_x(0)(V, V)
\]
because it holds for any $V$. Thus, the metrics $g(t)$ are uniformly equivalent to $g(0)$.

Now, we have the well-defined integral:
\[
g_x(T) - g_x(0) = \int_{0}^{T}\frac{\partial}{\partial t} g_x (t) d t.
\]
We say that this integral is well-defined because of two reasons. Firstly, as long as the metrics are smooth, the integral exists. Secondly, the integral is absolutely integrable. Based on the norm inequality induced by $g(0)$, one has
\[
|g_x(T) - g_x(t)|_{g(0)} \leq \int_{t}^{T}\left|\frac{\partial}{\partial t} g_x (t)\right|_{g(0)} d t.
\]
For each $x \in \mathcal{M}$, the above integral will approach to zero as $t$ approaches $T$. Because $\mathcal{M}$ is compact, the metrics $g(t)$ converge uniformly to a continuous metric $g(T)$ which is uniformly equivalent to $g(0)$ on $\mathcal{M}$. Moreover, we can show that
\[
e^{-C} g_x(0) \leq g_x(T) \leq e^C g_x(0).
\]

\section{Proof for the Information Geometry}
\label{app3}

\subsection{Proof for Theorem~\ref{thm7}}
\label{app31}

The linearly nearly Euclidean divergence can be defined between two nearby points $\boldsymbol{\xi}$ and $\boldsymbol{\xi}'$, where the first derivative of the linearly nearly Euclidean divergence w.r.t. $\boldsymbol{\xi}'$ is:
\[
\begin{aligned}
&\partial_{\boldsymbol{\xi}'} D_{LNE}[\boldsymbol{\xi}':\boldsymbol{\xi}] \\
&= \sum_i \left[\partial_{\boldsymbol{\xi}'} \frac{1}{\tau^2}\log\cosh(\tau\xi'_i)- \partial_{\boldsymbol{\xi}'} \frac{1}{\tau^2}\log\cosh(\tau\xi_i)- \frac{1}{\tau}\partial_{\boldsymbol{\xi}'}(\xi'_i-\xi_i)\tanh(\tau\xi_i)\right] \\
& = \sum_i \partial_{\boldsymbol{\xi}'} \frac{1}{\tau^2}\log\cosh(\tau\xi'_i) - \frac{1}{\tau}\tanh(\tau\boldsymbol{\xi}).
\end{aligned}
\]
The second derivative of the linearly nearly Euclidean divergence w.r.t. $\boldsymbol{\xi}'$ is:
\[
\partial^2_{\boldsymbol{\xi}'} D_{LNE}[\boldsymbol{\xi}':\boldsymbol{\xi}] =\sum_i \partial^2_{\boldsymbol{\xi}'} \frac{1}{\tau^2}\log\cosh(\tau\xi'_i).
\]
We deduce the Taylor expansion of the linearly nearly Euclidean divergence at $\boldsymbol{\xi}'=\boldsymbol{\xi}$:
\[
\begin{aligned}
D_{LNE}[\boldsymbol{\xi}':\boldsymbol{\xi}] &\approx D_{LNE}[\boldsymbol{\xi}:\boldsymbol{\xi}]+\left(\sum_i \partial_{\boldsymbol{\xi}'} \frac{1}{\tau^2}\log\cosh(\tau\xi'_i) - \frac{1}{\tau}\tanh(\tau\boldsymbol{\xi})\right)^\top \bigg|_{\boldsymbol{\xi}'=\boldsymbol{\xi}} d\boldsymbol{\xi} \\
&+\frac{1}{2}d\boldsymbol{\xi}^\top \left(\sum_i \partial^2_{\boldsymbol{\xi}'} \frac{1}{\tau^2}\log\cosh(\tau\xi'_i)\right) \bigg|_{\boldsymbol{\xi}'=\boldsymbol{\xi}} d\boldsymbol{\xi}\\
& =0 + 0 + \frac{1}{2\tau^2}d\boldsymbol{\xi}^\top \partial \left[ \frac{\partial\cosh(\tau\boldsymbol{\xi})}{\cosh(\tau\boldsymbol{\xi})}\right] d\boldsymbol{\xi} \\
&= \frac{1}{2\tau^2}d\boldsymbol{\xi}^\top  \frac{\partial^2\cosh(\tau\boldsymbol{\xi})\cosh(\tau\boldsymbol{\xi})-\partial\cosh(\tau\boldsymbol{\xi})\partial\cosh(\tau\boldsymbol{\xi})^\top}{\cosh^2(\tau\boldsymbol{\xi})} d\boldsymbol{\xi} \\
&= \frac{1}{2\tau^2}d\boldsymbol{\xi}^\top  \left(\frac{\partial^2\cosh(\tau\boldsymbol{\xi})}{\cosh(\tau\boldsymbol{\xi})} - \tau^2 \left[\frac{\sinh(\tau\boldsymbol{\xi})}{\cosh(\tau\boldsymbol{\xi})}\right]\left[\frac{\sinh(\tau\boldsymbol{\xi})}{\cosh(\tau\boldsymbol{\xi})}\right]^\top\right) d\boldsymbol{\xi} \\
&=\frac{1}{2}\sum_{i,j} \delta_{ij} -\left[\tanh(\tau\boldsymbol{\xi})\tanh(\tau\boldsymbol{\xi})^\top\right]_{ij} d\xi_i d\xi_j.
\end{aligned}
\]

\subsection{Proof for Lemma~\ref{lem5}}
\label{app32}

We would to know in which direction minimizes the loss function with the constraint of the linearly nearly Euclidean divergence, so that we do the minimization:
\[
d\boldsymbol{\xi}^{*}=\underset{d\boldsymbol{\xi} \text { s.t. } D_{LNE}[\boldsymbol{\xi}:\boldsymbol{\xi}+d\boldsymbol{\xi}]=c}{\arg \min } \mathbb{L}(\boldsymbol{\xi}+d\boldsymbol{\xi})
\]
where $c$ is the constant. The loss function descends along the manifold with constant speed, regardless the curvature.

Now, we write the minimization in Lagrangian form. Combined with Theorem~\ref{thm7}, the linearly nearly Euclidean divergence can be approximated by its second order Taylor expansion. Approximating $\mathbb{L}(\boldsymbol{\xi}+d\boldsymbol{\xi})$ with it first order Taylor expansion, we get:
\[
\begin{aligned}
d\boldsymbol{\xi}^* &=\underset{d\boldsymbol{\xi}}{\arg \min}\ \mathbb{L}(\boldsymbol{\xi}+d\boldsymbol{\xi}) + \lambda\left(D_{LNE}[\boldsymbol{\xi}:\boldsymbol{\xi}+d\boldsymbol{\xi}]-c\right) \\
&\approx \underset{d\boldsymbol{\xi}}{\arg \min}\ \mathbb{L}(\boldsymbol{\xi}) + \partial_{\boldsymbol{\xi}}\mathbb{L}(\boldsymbol{\xi})^\top d\boldsymbol{\xi}+\frac{\lambda}{2}d\boldsymbol{\xi}^\top \bar{g}_{ij}(t) d\boldsymbol{\xi}-c\lambda.
\end{aligned}
\]
To solve this minimization, we set its derivative w.r.t. $d \boldsymbol{\xi}$ to zero:
\[
\begin{aligned}
0 &=\frac{\partial}{\partial d\boldsymbol{\xi}} \mathbb{L}(\boldsymbol{\xi})+\partial_{\boldsymbol{\xi}} \mathbb{L}(\boldsymbol{\xi})^\top d\boldsymbol{\xi}+\frac{\lambda}{2} d\boldsymbol{\xi}^\top \left[\delta_{ij} - \tanh(\tau\boldsymbol{\xi}(t))\tanh(\tau\boldsymbol{\xi}(t))^\top\right] d\boldsymbol{\xi}-c\lambda \\
&=\partial_{\boldsymbol{\xi}} \mathbb{L}(\boldsymbol{\xi})+\lambda \left[\delta_{ij} - \tanh(\tau\boldsymbol{\xi}(t))\tanh(\tau\boldsymbol{\xi}(t))^\top\right] d\boldsymbol{\xi} \\
d\boldsymbol{\xi} &=-\frac{1}{\lambda} \left[\delta_{ij} - \tanh(\tau\boldsymbol{\xi}(t))\tanh(\tau\boldsymbol{\xi}(t))^\top\right]^{-1} \partial_{\boldsymbol{\xi}} \mathbb{L}(\boldsymbol{\xi})
\end{aligned}
\]
where a constant factor $1/\lambda$ can be absorbed into learning rate. Up to now, we get the optimal descent direction, i.e., the opposite direction of gradient which takes into account the local curvature defined by $\left[\delta_{ij} - \tanh(\tau\boldsymbol{\xi}(t))\tanh(\tau\boldsymbol{\xi}(t))^\top\right]^{-1}$.

\subsection{Proof for Corollary~\ref{cor4}}
\label{app33}

\begin{definition}
	\label{strict}
	For $\boldsymbol{A} \in \mathcal{R}^{n \times n}$, $\boldsymbol{A}$ is called as the strictly diagonally-dominant matrix when it satisfies
	\[
	\big|a_{ii}\big| > \sum^n_{j=1,j\neq i} \big|a_{ij}\big|,\;\;\; i=1,2,\ldots,n.
	\]
\end{definition}

\begin{definition}
	\label{nonsingular}
	If $\boldsymbol{A} \in \mathcal{R}^{n \times n}$ is a strictly diagonally-dominant matrix, then $\boldsymbol{A}$ is a nonsingular matrix together.
\end{definition}

Subsequently, we can consider the inverse matrix of the metric $\bar{g}(t)$. Due to the strictly diagonally-dominant feature in Definition~\ref{strict} and Definition~\ref{nonsingular}, we can approximate $\left[\delta_{ij} - \tanh(\tau\boldsymbol{\xi}(t))\tanh(\tau\boldsymbol{\xi}(t))^\top\right]^{-1}$. As for we can also ignore the fourth-order small quantity $\sum O(\rho_a \rho_b \rho_c \rho_d)$ because of the characteristic of the strictly diagonally-dominant, so that
\[
\begin{aligned}
&\left[\delta_{ij} - \tanh(\tau\boldsymbol{\xi}(t))\tanh(\tau\boldsymbol{\xi}(t))^\top\right]\left[\delta_{ij} + \tanh(\tau\boldsymbol{\xi}(t))\tanh(\tau\boldsymbol{\xi}(t))^\top\right] \\
&= \begin{bmatrix} 1-\rho_1 \rho_1 & -\rho_1 \rho_2 &\cdots\\ -\rho_2 \rho_1 & 1-\rho_2 \rho_2 &\cdots \\
\vdots&\vdots&\ddots
\end{bmatrix} \begin{bmatrix} 1+\rho_1 \rho_1 & \rho_1 \rho_2 &\cdots\\ \rho_2 \rho_1 & 1+\rho_2 \rho_2 &\cdots \\
\vdots&\vdots&\ddots
\end{bmatrix} \\
& = \begin{bmatrix} 1-\sum O(\rho_a \rho_b \rho_c \rho_d)& \rho_1 \rho_2-\rho_1 \rho_2 -\sum O(\rho_a \rho_b \rho_c \rho_d) &\cdots\\ -\rho_2 \rho_1 + \rho_2 \rho_1 -\sum O(\rho_a \rho_b \rho_c \rho_d) & 1-\sum O(\rho_a \rho_b \rho_c \rho_d) &\cdots \\
\vdots&\vdots&\ddots
\end{bmatrix} \approx \boldsymbol{I}.
\end{aligned}
\]
Note that the Euclidean metric $\delta_{ij}$ is equal to the identity matrix $\boldsymbol{I}$.

\section{Proof for Corollary~\ref{cor3}}
\label{app4}

	Based on Equation~(\ref{deturck2}), we know
\[
\begin{aligned}
\frac{\partial}{\partial t} (d(t)-d_0)=&\Delta (d(t)-d_0)+\operatorname{Rm}*(d(t)-d_0) \\
&+F_{\bar{g}^{-1}} * \nabla^{\bar{g}_0} (d(t)-d_0) * \nabla^{\bar{g}_0} (d(t)-d_0) \\ &+\nabla^{\bar{g}_0}\left(G_{\Gamma(\bar{g}_0)} * (d(t)-d_0) * \nabla^{\bar{g}_0} (d(t)-d_0)\right).
\end{aligned}
\]
Followed by Lemma~\ref{lem4} and Theorem~\ref{thm5}, we further obtain
\[
\begin{aligned}
\frac{\partial}{\partial t} \| d(t) -d_{0} \|_{L^{2}}^{2}=& 2\left(\Delta (d(t)-d_0)+\operatorname{Rm}*(d(t)-d_0), d(t)-d_{0}\right)_{L^{2}} \\
&+\left(F_{\bar{g}^{-1}} * \nabla^{\bar{g}_0} (d(t)-d_0) * \nabla^{\bar{g}_0} (d(t)-d_0), d(t)-d_{0}\right)_{L^{2}} \\
&+\left(\nabla^{\bar{g}_0}\left(G_{\Gamma(\bar{g}_0)} * (d(t)-d_0) * \nabla^{\bar{g}_0} (d(t)-d_0)\right), d(t)-d_{0}\right)_{L^{2}} \\
&+\left(d(t)-d_{0}, \frac{\partial}{\partial t} d_{0}(t)\right)_{L^{2}}+\int_{\mathcal{M}}\left(d(t)-d_{0}\right) *\left(d(t)-d_{0}\right) * \frac{\partial}{\partial t} d_{0}(t) \mathrm{d} \mu \\
\leq &-2 \alpha_{\bar{g}_0}\left\|\nabla^{\bar{g}_0}\left(d(t)-d_{0}\right)\right\|_{L^{2}}^{2} \\
&+C\left\|\left(d(t)-d_{0}\right)\right\|_{L^{\infty}}\left\|\nabla^{\bar{g}_{0}}\left(d(t)-d_{0}\right)\right\|_{L^{2}}^{2} \\
&+\left\|\frac{\partial}{\partial t} d_{0}(t)\right\|_{L^{2}}\left\|d(t)-d_{0}\right\|_{L^{2}} \\
\leq &\left(-2 \alpha_{\bar{g}_{0}}+C \cdot \epsilon\right)\left\|\nabla^{\bar{g}_{0}}\left(d(t)-d_{0}\right)\right\|_{L^{2}}^{2}.
\end{aligned}
\]
Let $\epsilon$ be small enough that $-2 \alpha_{\bar{g}_{0}}+C \cdot \epsilon < 0$ holds, we can find
\[
\frac{\partial}{\partial t} \| d(t) -d_{0} \|_{L^{2}}^{2} \leq -C \left\|\nabla^{\bar{g}_{0}}\left(d(t)-d_{0}\right)\right\|_{L^{2}}^{2}
\]
holds.

\section{Experiment}
\label{app5}

\begin{figure}[thbp]
	\centering
	\subfigure[CIFAR10]
	{\includegraphics[width=0.85\textwidth]{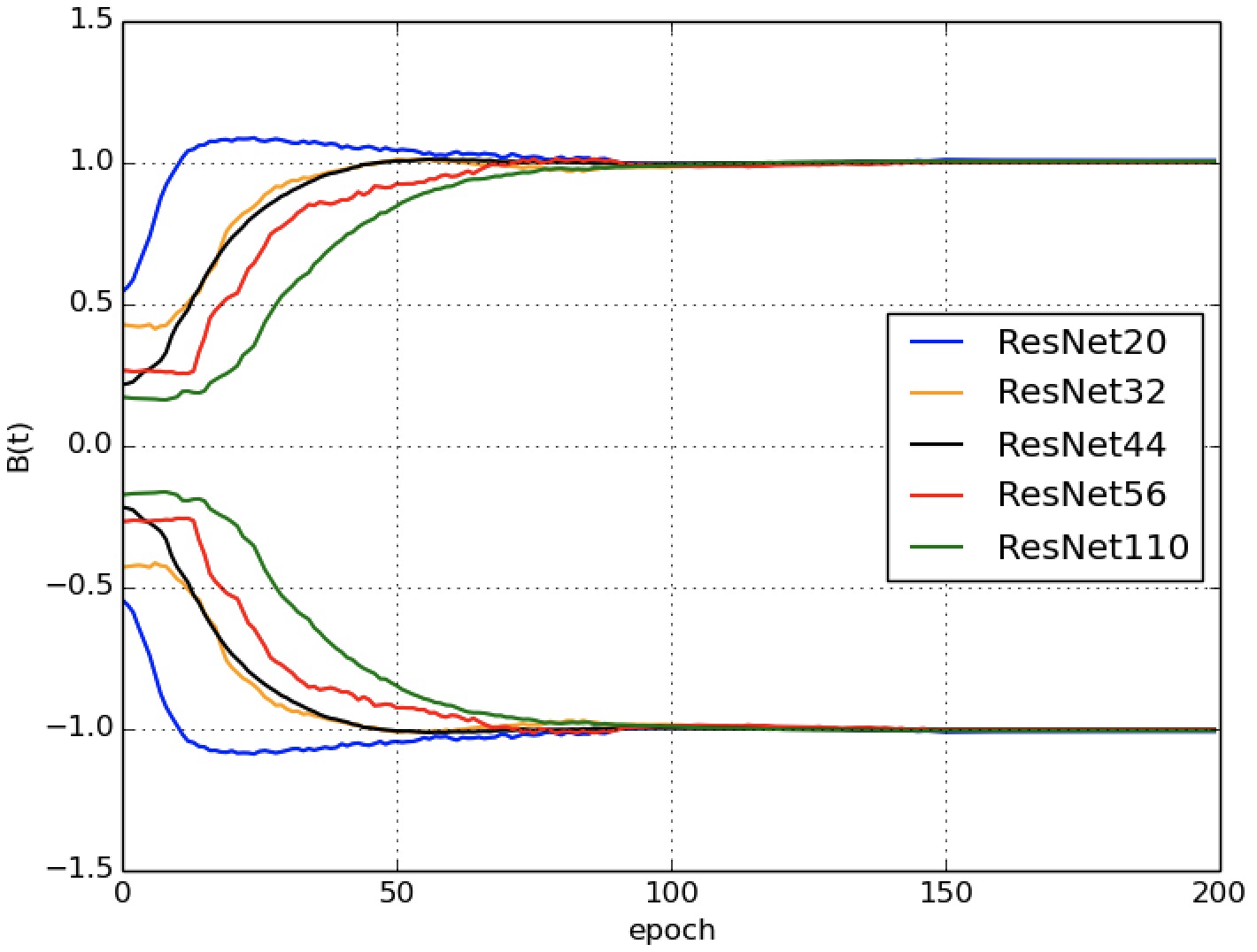}}
	\subfigure[CIFAR100]
	{\includegraphics[width=0.85\textwidth]{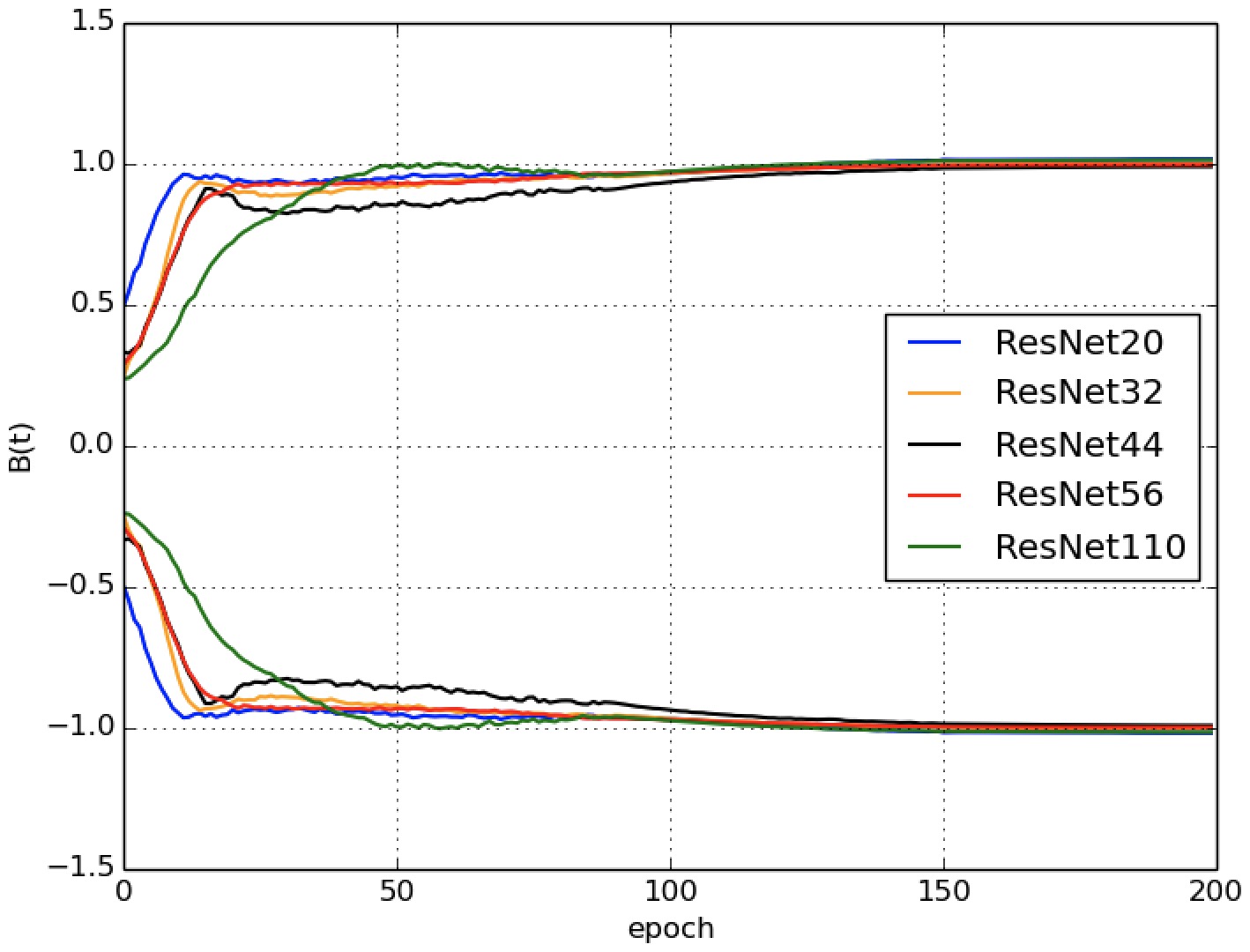}}
	\caption{The evolution of metrics $\bar{g}(t)$ by the radius of a ball with the epoch of training process. Note that we use radius $1$ as the calibration of a linearly nearly Euclidean metric.}
	\label{curvature}
\end{figure}

{\bf CIFAR datasets.} The two CIFAR datasets \cite{krizhevsky2009learning} consist of natural color images with 32$\times$32 pixels, respectively 50,000 training and 10,000 test images, and we hold out 5,000 training images as a validation set from the training set. CIFAR10 consists of images organized into 10 classes and CIFAR100 into 100 classes. We adopt a standard data augmentation scheme (random corner cropping and random flipping) that is widely used for these two datasets. We normalize the images using the channel means and standard deviations in preprocessing.

{\bf Settings.} We set total training epochs as 200 where the learning strategy is lowered by 10 times at epoch 80, 150, and 190, with the initial 0.1. The learning strategy is a weight decay of 0.0001, a batch size of 128, SGD optimization. On CIFAR10 and CIFAR100 datasets, we apply ResNet20, ResNet32, ResNet44, ResNet56 and ResNet110 models~\citep{he2016deep} to observe the evolution of neural manifold, i.e., the convergence of metrics depended on time. As far as we define the metric $\bar{g}(t)$, we can use the length $|ds^2|=\sqrt{\sum_{i,j} \bar{g}_{ij}(t) d\xi_i d\xi_j}$ to intuitively reflect the change of metrics. Specifically, we define a ball whose radius is equal to $|ds^2|$:
\begin{equation}
B_r(t):=\left\{r=\sqrt{\sum_{i,j} \bar{g}_{ij}(t) d\xi_i d\xi_j}\right\}.
\end{equation}

{\bf Details.} We embed the linearly nearly Euclidean manifold into a neural network, which means that a neural networks uses Corollary~\ref{cor4} for back-propagation. No other parts of the neural network need to be modified.

{\bf Neural Network Behavior.} By observing the change of the ball in Figure~\ref{curvature}, we can know the change of the metric. Through simple observation, metrics $\bar{g}(t)$ on CIFAR10 converge in about 100 epochs and metrics $\bar{g}(t)$ on CIFAR100 converge in about 150 epochs. For CIFAR10, metrics $\bar{g}(t)$ in ResNet32 and ResNet44 seem to converge the fastest. For CIFAR100, metrics $\bar{g}(t)$ in ResNet110 seem to converge the fastest. In general, experiments show that all metrics in neural manifolds converge to a linearly nearly Euclidean metric. It is consistent with the evolution of Ricci-DeTurck flow in Section~\ref{chapter3}.

{\bf Remark.} For a neural network that is specified by connection weights, the set of all such weighs forms a manifold. On linearly nearly Euclidean manifolds, when we use the weakly approximated gradient flow to learn a neural network ($\boldsymbol{\xi}(t)$ is composed of weights), we observe the evolution of its metric is consistent with the convergent behavior under Ricci-DeTurck flow. Consequently, the training of a neural manifold is gradually regular, whose metric is stable and eventually converges to the linearly nearly Euclidean metric.

\end{document}